\documentclass[letterpaper, 10 pt, conference]{ieeeconf}

\IEEEoverridecommandlockouts
\usepackage[T1]{fontenc}    
\usepackage{hyperref}       
\usepackage{url}            
\usepackage{booktabs}       
\usepackage{microtype}      
\usepackage{soul}
\usepackage{comment}
\usepackage{adjustbox}
\usepackage{subcaption}
\usepackage{multirow}
\usepackage{float}
\usepackage{xcolor}
\hypersetup{
	colorlinks,
	linkcolor={red!50!black},
	citecolor={blue!50!black},
	urlcolor={blue!80!black}
}
\usepackage{hyperref}
\usepackage{graphicx}      
\usepackage{tabulary}
\usepackage{amsfonts}
\usepackage{amsmath}
\usepackage{amssymb}
\usepackage{pgf,tikz}
\usepackage{tabularx}
\usepackage{float}
\usepackage{wrapfig}
\usepackage{cite}
\usepackage{stfloats}
\usepackage{nicefrac}
\usepackage{nicematrix}
\usepackage{setspace}

\usepackage{booktabs}
\usepackage{multicol}
\usepackage{lipsum}
\usepackage{algorithm}
\usepackage{algpseudocode}

\let\labelindent\relax
\usepackage{enumitem} 
\usepackage{bbm} 
\usepackage{bm} 
\usepackage[algo2e,ruled,vlined,linesnumbered,resetcount,norelsize]{algorithm2e}

\SetCommentSty{mycommfont}

\SetKwRepeat{Do}{do}{until}

\makeatletter
\renewcommand*\env@matrix[1][\arraystretch]{
	\edef\arraystretch{#1}%
	\hskip -\arraycolsep
	\let\@ifnextchar\new@ifnextchar
	\array{*\c@MaxMatrixCols c}}
\makeatother


\newcounter{parentnumber}

\newtheorem{theorem}{Theorem}[section]

\newtheorem{lemma}[theorem]{Lemma}

\newtheorem{remark}[theorem]{Remark}

\newtheorem{assumption}{Assumption}

\newcommand{\col}{{\rm col\;}}

\newcommand{\until}[1]{\{1,\cdots,#1\}}
\newcommand{\longthmtitle}[1]{\mbox{} \emph{(#1):}}

\newcommand\oprocendsymbol{\hbox{$\square$}}
\newcommand\oprocend{\relax\ifmmode\else\unskip\hfill\fi\oprocendsymbol}

\renewcommand{\baselinestretch}{.98}

\def\qed{ \rule{.1in}{.1in}}

\title{D3G: Learning Multi-robot Coordination from Demonstrations}

\author{%
Yizhi Zhou, Wanxin~Jin and Xuan~Wang 
 \thanks{
	Work supported by Army Research Office (W911NF-22-2-0242) and NSF (2332210). George Mason University.
 X. Wang and Y. Zhou are with the
		Department of Electrical and Computer Engineering, George Mason
		University. Wanxin Jin is with the School for Engineering of Matter, Transport, and Energy, Arizona State University. 
        Point of contact: {\tt xwang64@gmu.edu}.
  }
}

\begin{document}
	
\maketitle

\begin{abstract}


This paper develops a new Distributed approach for solving the inverse problem of a Differentiable Dynamic Game (D3G), which enables robots to learn multi-robot coordination from given demonstrations. 
We formulate multi-robot coordination as the Nash equilibrium of a parameterized dynamic game, where the behavior of each robot is dictated by an objective function that also depends on the behavior of its neighboring robots. The coordination thus can be adapted by tuning the parameters of the objective and the local dynamics of each robot. The proposed algorithm enables each robot to automatically tune such parameters in a distributed and coordinated fashion --- only using the data of its neighbors without global information. Its key novelty is the development of a distributed solver for a diff-KKT condition that can enhance scalability and reduce the computational load for gradient computation. We test the proposed algorithm in simulation with heterogeneous robots given different task configurations. The results demonstrate its effectiveness and generalizability for learning multi-robot coordination from demonstrations.

\end{abstract}

\section{Introduction}
The control and coordination of large-scale multi-robot systems have long been viewed as a challenging problem, due to the need for robots to make sequential and coordinated decisions~\cite{RY-AM-TEW:19}. 
Dynamic game theory provides an accessible framework for modeling the interaction among multiple robots, whose behaviors are dictated by their local observations and coupled objective functions~\cite{ZK-YZ-BT:21,wang2022consensus}. To ensure that robot interactions lead to meaningful coordinated behavior, objective functions must be carefully designed, which however is technically non-trivial, and mostly relies on heuristic trial-and-error. In contrast, specifying/demonstrating desired robot behaviors is much more intuitive. This has motivated the research of learning objective functions from demonstrations also known as inverse dynamic game (IDG)~\cite{LX-ASC-BPA:19}.
In the counterpart problem for a single robot case, many tools and methods are available, ranging from imitation learning~\cite{hussein2017imitation}, learning from demonstrations~\cite{jin2022learning}, to, most recently, differentiable optimal control~\cite{jin2021safe,jin2020pontryagin}. However, scalable solutions to address the aforementioned challenges in multi-robot systems are quite limited, mainly due to the dimensionality of the problem that quickly grows with the number of robots. In this paper, we propose a new Distributed Differentiable Dynamic Game (D3G) framework for solving IDG, where each robot automatically learns its objective function in a \textit{distributed} and coordinated fashion --- only using the data of its neighbors without global information. At the core of our algorithm is a distributed solver that leverages the \textit{differentiability} of the KKT condition (diff-KKT) to enhance scalability and reduce the computational load for gradient computation. A conceptual diagram of D3G is in Fig.~\ref{ALG_demo}.

\begin{figure}[t]
	\vspace{-1ex}
	\centering
	\includegraphics[width=0.48\textwidth]{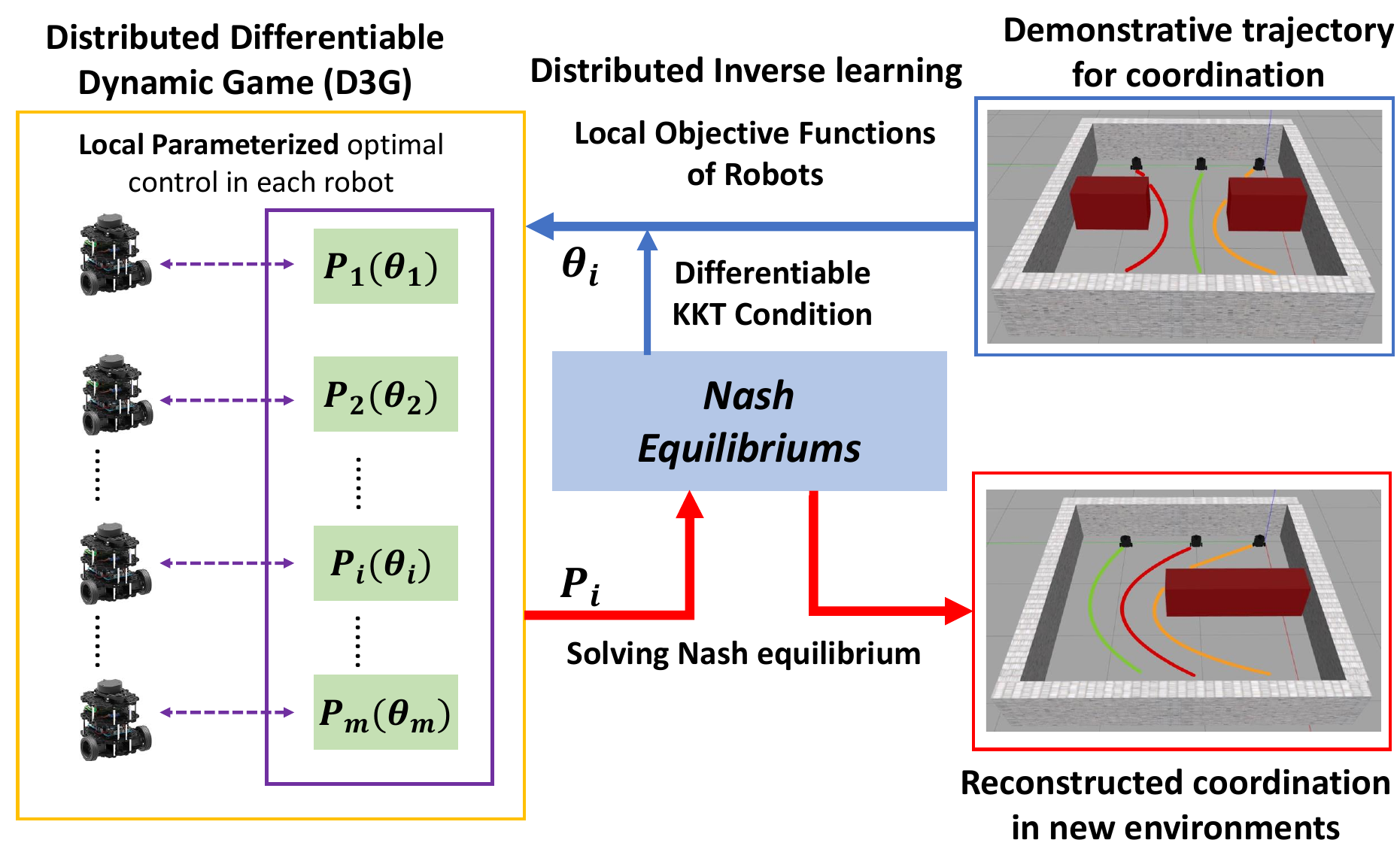}
	\caption{\small Each robot possesses a local optimal control $\mathbf{P}_i$, which together constitutes a dynamic game. The Nash equilibrium of the game reconstructs robot coordination. Problem of interest: Distributed inverse learning (blue) of parameterized objective functions from demonstration for robot coordination. The learned objective is generalizable (red) to new environments.\\
 }
	\label{ALG_demo}
\end{figure}



\subsubsection*{Related Work}
Learning from demonstrations can be formulated as a problem of inverse optimal control (IOC), also known as \textit{Inverse reinforcement learning}, seeking to learn an objective function of a decision-making agent from expert demonstrations~\cite{ab2020inverse}.
One type of method for solving IOC directly minimizes the residual of the optimality (KKT) conditions by assuming that the demonstration is optimal and fulfills these conditions~\cite{EP-VNA-TM:17}.
Another common approach is built upon a bi-level structure, containing a forward loop and an inverse loop. The forward loop solves a standard optimal control problem with the current objective estimate. Available methods for this include dynamic programming~\cite{BD:12}, trajectory optimizations~\cite{RM-PM:17}, and reinforcement learning~\cite{EJ:21}. The inverse loop updates the objective estimate such that a trajectory of the forward loop matches the demonstrations by minimizing certain losses. 
Different methods for IOC vary in how to accommodate the forward and inverse loops ~\cite{jin2020pontryagin,RND-BJA-ZMA:06,ZBD-ALB-JAD-AK:08}
, and also how to define loss functions, such as least square~\cite{jin2020pontryagin}, maximum margin~\cite{RND-BJA-ZMA:06}, maximum entropy~\cite{ZBD-ALB-JAD-AK:08}.

Dynamic game generalizes optimal control to a multi-robot setup, where each robot's objective functions depend on its own action and the actions of other robots over time. Addressing such sequential decision-making processes often involves treating agents' entire state and action trajectories as variables~\cite{YM-HG:17}. The set of robots' planned trajectories, when no one can improve its objective function by changing its behaviors, constitutes the solution to the game called open-loop Nash equilibrium~\cite{FD-TJ:91}.
Common approaches to obtaining a Nash equilibrium include: designing an algorithm whose dynamics asymptotically converge to the desired Nash equilibrium~\cite{tatarenko2020geometric,BA-SW:04}. To satisfy dynamics constraints,~\cite{SF-PL:16} 
introduced a projection operator that restricts the gradient flow to a feasible set, ensuring compliance with an agent's local constraints. 
An alternative approach is to compute Nash equilibrium directly from its holding conditions~\cite{BA:10}, which can be done by generalizing the Pontryagin's Maximum Principle/KKT condition~\cite{LSM:06} to a game theoretical setup.

Analogous to IOC, dynamic games also have their \textit{inverse problem}, i.e., given robots collective trajectories satisfying a Nash equilibrium, how to inversely learn the objective functions the robots aim to optimize~\cite{MTL-FJJ-PT:17,LX-ASC-BPA:19}. 
Existing works for solving inverse games have three main categories. 
The first category aims to solve the inverse game by applying derivative-free filter-based approaches built upon Bayesian inference~\cite{LC-SS-MZ:21, PT:21}, which however has high sample complexity and requires exact observations of state.
The second category solves the inverse game by equilibrium-constrained maximum-likelihood estimation (MLE), which uses the optimality conditions of the open-loop Nash equilibrium, to formulate a constrained optimization problem~\cite{PL-DR-CS:2021, MTL-FJJ-PT:17}. This type of method can explicitly handle noisy data and partial observations.
The third category follows the minimization of residual methods~\cite{CR-JF:17,CK-XL:22}, which seek to minimize the residual of the first-order necessary conditions of an open-loop Nash equilibrium. These works are further extended in~\cite{AC-LA:20,FD-RE-LD:2020,LX-LA-J:2023} to involve state and input constraints. 

The approach proposed in this paper is similar to the ones in the last category~\cite{CK-XL:22,LX-LA-J:2023}. However, we note that existing methods for solving inverse dynamic games rely on a centralized process, where the forward loop and inverse loop are solved using the global information of all robots. Consequently,  the computation and communication complexity grows exponentially with the number of robots and planning horizons. While there exist distributed approaches for solving the forward problem~\cite{YM-HG:17, tatarenko2020geometric,ZY-HB-MZ-RW:21}, the scalability challenge remains for addressing the inverse problem. 
This motivates us to develop a distributed inverse dynamics game framework for \textit{scaling} the complexity of global computation and communication into a local and coordinated approach.



\subsubsection*{Statement of Contributions}

We study the problem of learning multi-robot coordination from demonstration by
formulating it as a differentiable dynamic game. Each robot in the game satisfies its dynamics and optimizes a coupling objective function. Both the dynamics and objective of each robot are unknown and learnable. We propose a D3G framework to inversely solve the dynamic game by minimizing the
mismatch between the predicted multi-robot trajectories of the dynamic game and the given demonstrations. The learning update of D3G is based on local gradient descent. This allows a fully distributed algorithm design, where each robot uses the diff-KKT condition to compute its parameter update, by only using the data of its neighbors without global information. The effectiveness and scalability of D3G are verified using two types of robots given four different task configurations.






\section{Preliminaries and Problem Formulation}

\subsection{Parametric  Dynamic Game for Multi-robot Coordination}
Consider a system of $m$ robots. Suppose each robot solves its own optimal control problem \ref{eq_OCP} parameterized by a vector $\theta_i\in\mathbb{R}^{r_i}$ as follows: 
\begin{align}
		\nonumber
			\min_{\mathbf{u}_i} & \quad \mathcal{J}_i(\theta_i)=\sum_{t=0}^{T-1} c_i^t(x_i^t, u_i^t, {x_{\mathcal{N}_i}^t}, \theta_i)+h_i(x_i^T, x_{\mathcal{N}_i}^T, \theta_i),
			\\  
			\text{s.t.} & \quad x_i^{t+1}=f_i(x_i^t, u_i^t, \theta_i) \quad \text{given} ~~x_i^0. \nonumber
\label{eq_OCP}
\tag{$\displaystyle \mathbf{P}_i(\theta_i)$}
\end{align}
Here, for robot $i$, $x_i^t \in\mathbb{R}^{n_i}$ and $u_i^t \in\mathbb{R}^{m_i}$ are the robot's state and control input at each time step $ t\in\{0,1,2,\cdots,T\}$;  $\mathbf{x}_{i}=\{{x}_{i}^{0},\cdots, {x}_{i}^{T}\}$ and $\mathbf{u}_{i}=\{{u}_{i}^{0},\cdots, {u}_{i}^{T-1}\}$; $f_i(\cdot)\in\mathbb{R}^{n_i}$ is the robot dynamics; $\mathcal{J}_i(\cdot)\in\mathbb{R}$ is the local control objective function with $c_i^t(\cdot)\in\mathbb{R}$ and $h_i(\cdot)\in\mathbb{R}$ denoting the running and final costs, respectively. 
To characterize the fact that connected robots have coordinated behaviors, their objective functions are set to be coupled, i.e., $\mathcal{J}_i(\cdot)$ depends not only on the state/input of robot $i$, but also on that of its neighbors, denoted by $x_{\mathcal{N}_i}^t=\{x_j^t~|~j\in\mathcal{N}_i\}$, with $\mathcal{N}_i$ being the neighbor set of robot $i$. 
The neighborhoods of robots define the communication topology $\mathbb{G}$ across the whole system, whose vertices are associated with the robots. We assume $\mathbb{G}$  is undirected.
Further define $\bm{\xi}_i = \{\mathbf{x}_{i}, \mathbf{u}_{i}\}$, which represents the full trajectory of robot $i$ for all time steps.



Since each robot only makes local observations, the collection of optimal control problems \ref{eq_OCP} across all robots forms a \textit{general-sum} dynamic game $\mathbf{P}(\Theta)$ parameterized by $\Theta=\col\{\theta_1,\cdots,\theta_m\}\in\mathbb{R}^{\sum_i^m r_i}$. 
Given the objective functions $\mathcal{J}_i(\theta_i)$ to be mutually coupled, the `forward' (v.s. inverse) problem of the game $\mathbf{P}(\Theta)$ is to obtain a set of state-input-trajectories $\bm{\xi}^{\star}_{i}(\Theta)=\{{\mathbf{x}^{\star}_{i}}(\Theta),{\mathbf{u}^{\star}_{i}}(\Theta)\}$ for all $i\in\until{m}$, called \textit{open-loop Nash Equilibrium (N.E.)}\footnote{In this paper, we refer to N.E. as an open-loop Nash equilibrium, in contrast to the feedback Nash equilibrium~\cite[Chapter 3]{BT-OGJ:98}.}, satisfying: 
\begin{equation}
\begin{aligned}
	&\mathcal{J}_i(\bm{\xi}_{i}^{\star}(\Theta),~\bm{\xi}_{\mathcal{N}_i}^{\star}(\Theta),~\theta_i)\le\mathcal{J}_i(\bm{\xi}_{i},~\bm{\xi}_{\mathcal{N}_i}^{\star}(\Theta),~\theta_i)\\
	&\qquad\text{s.t.}\qquad \bm{\xi}_i\in\Xi_{i} (\theta_i).
\end{aligned}
\tag{N.E.}
\label{eq_NashP}
\end{equation}
where $\Xi_{i}(\theta_i)$ is the set of all feasible trajectories of robot $i$ satisfying its initial condition and system dynamics.  $\Xi_{i}(\theta_i)$ is a function of $\theta_i$ because the dynamics $f_i(\cdot)$ is parameterized by $\theta_i$. 
We use the \eqref{eq_NashP} of $\mathbf{P}(\Theta)$ to characterize distributed multi-robot coordination, where each robot determines its trajectory $\bm{\xi}^{\star}_{i}(\Theta)$ based on the local information of its neighboring robots. 
$\bm{\xi}^{\star}_{i}(\Theta)$ is a function of tunable $\Theta$.



\subsection{Problem Formulation}
While lots of effort has been given to solve the `forward' problem of $\mathbf{P}(\Theta)$, i.e., calculating its \eqref{eq_NashP} given robots' objective functions, this work focuses on the `inverse' problem: \textit{Which objective functions (the parameters for $\Theta$) can reconstruct desired multi-robot coordination strategies that are aligned with given demonstrations.} 

To this end, we first introduce the following assumption.
\begin{assumption}\label{Ass_func}
Both $\mathcal{J}_i(\cdot)$ and $f_i(\cdot)$ are twice differentiable. Given other variables being fixed, the cost function $\mathcal{J}_i(\cdot)$ is strictly convex on $x_i$ and $u_i$. The feasible trajectory set $\Xi_{i}(\theta_i)$ is convex and bounded.  
\end{assumption}
\smallskip

Assumption \ref{Ass_func} ensures the existence and uniqueness of a pure \eqref{eq_NashP} for $\mathbf{P}(\Theta)$~\cite[Theorem 4.3]{BT-OGJ:98}. It imposes some mild conditions on $f_i(\cdot)$ and $\mathcal{J}_i(\cdot)$, which are common in the existing literature for game-theoretic studies of multi-robot systems
~\cite{tatarenko2020geometric, YM-HG:17, ZY-HB-MZ-RW:21}
. These conditions generally hold for physical models of simple mobile robots and regular cost functions such as distance to the goal. In the case that $\Xi_{i}(\theta_i)$ is unbounded, the existence and uniqueness can still be guaranteed~\cite[Corollary 4.2]{BT-OGJ:98} if we further assume $\mathcal{J}_i(\cdot)\to\infty$ as $|x_i|$ or $|u_i|\to\infty$. This holds for most cost functions.

\smallskip
\noindent\textbf{Problem of interest}: 
Given the demonstrations of robot trajectories $\{\bm{\xi}_1^d,\bm{\xi}_2^d,\cdots, \bm{\xi}_m^d\}$, $d\in\until{D}$, that are associated with the \eqref{eq_NashP} of a game $\mathbf{P}(\Theta)$, with unknown $\Theta$. 
Suppose each robot $i$ locally knows $\mathbf{P}_i(\cdot)$ and $\bm{\xi}_i^d$. We aim to develop a fully distributed algorithm over $\mathbb{G}$ such that all robots jointly learn the parameter $\Theta^{\star}=\col\{\theta_1^{\star},\cdots,\theta_m^{\star}\}$ by minimizing the following loss function
	\begin{equation}
	\begin{aligned}
		\min_{\Theta=\col\{\theta_1,\cdots,\theta_m\}}~\sum_{i=1}^{m} {\mathcal{L}_i({\bm{\xi}^{\star}_i(\Theta)},\bm{\xi}_i^d)}.
	\end{aligned}
	\label{eq_problem}
	\end{equation}
    The loss function in each robot is defined as 
    \begin{equation}
\begin{aligned} 
	{\mathcal{L}_i({\bm{\xi}^{\star}_i(\Theta)},\bm{\xi}_i^d)}=\sum_{d=1}^{D}\|{\bm{\xi}^{\star}_i(\Theta)}-\bm{\xi}_i^d\|^2_2
\end{aligned}
\label{eq_lost}
\end{equation}

By minimizing (\ref{eq_problem}),  we learn a proper $\Theta^{\star}$, i.e., $\theta_i^{\star}$ for each robot, to best mimic/reproduce the demonstrations (from experts) using the \eqref{eq_NashP} of the parameterized game.
In the above definition of the loss (\ref{eq_lost}), we consider the robot's trajectories at each time instant to be equally important, but other definitions of the loss~\cite{jin2020pontryagin,RND-BJA-ZMA:06,ZBD-ALB-JAD-AK:08} are also applicable.


\section{Inverse Learning for Distributed Differential Dynamic Game} \label{Sec_Framekwork}


\subsection{Method Overview}

To solve the formulated problem, we develop a fully distributed learning paradigm, where each robot updates its own $\theta^{\star}_i$ for \ref{eq_OCP} using only its local data and neighboring communication. 
We are enlightened by local gradient descent to propose the following algorithm, 
\begin{equation}
\begin{aligned}
		&\theta_i^{k+1} = \theta_i^k-\eta^k\left.\frac{d{\mathcal{L}_i({\bm{\xi}^{\star}_i(\Theta)},\bm{\xi}_i^d)}}{d\theta_i}\right|_{\theta_i^k}
\end{aligned}
\label{eq_grad}
\end{equation}\noindent where $\eta^k$ is the learning rate. Compared with the global full gradient, local gradient descent requires stricter step sizes to ensure algorithm stability; however, it achieves significant computational tractability. Similar techniques are used in many machine learning methods, such as actor-critic methods, where the actor and critic models are updated in a decoupled manner~\cite{ZL-FT-AZ-CB-RLJ:22}.  
In addition, recall that the global and local loss functions defined in \eqref{eq_lost} and \eqref{eq_problem} are both non-negative. If the demonstrations and the generated trajectories can match perfectly, $\sum_{i-1}^m\mathcal{L}_i$ and $\mathcal{L}_i$ share the same minimizer at 0. The effectiveness of `local gradients' will be further justified by our experiments.

The implementation of update \eqref{eq_grad} is summarized in Algorithm \ref{Algorithm_DIDG}, and it relies on the following chain rule to compute the gradient. 
{\begin{equation}
\begin{aligned}
\left.\frac{d{\mathcal{L}_i({\bm{\xi}^{\star}_i(\Theta)},\bm{\xi}_i^d)}}{d\theta_i}\right|_{\theta_i^k}\!\!=\! 
	\left.\frac{\partial {\mathcal{L}_i({\bm{\xi}^{\star}_i(\Theta)},\bm{\xi}_i^d)}}{\partial\bm{\xi}^{\star}_i(\Theta)}\right|_{\bm{\xi}^{\star}_i(\Theta^k)}
	\!\cdot \!\left.{\frac{\partial\bm{\xi}^{\star}_i(\Theta)}{\partial\theta_i}}\right|_{\theta_i^k}.
\end{aligned}
\label{eq_gradchain}
\end{equation}}

\begin{algorithm2e}[t]
	\setstretch{1.0}
	\caption{Inverse Learning for Distributed Differential Dynamic Game, the local update for robot $i$..}
	\label{Algorithm_DIDG}
	\SetAlgoLined
	\textbf{Input} Demonstrations of trajectory $\bm{\xi}_i^d$.\\ 
	Initialize a random guess for $\theta_i^{k=0}$.\\
    \For{$k=0,1,2,\cdots$}
	{
	Compute {\small$\displaystyle\frac{\partial {\mathcal{L}_i({\bm{\xi}^{\star}_i(\Theta)},\bm{\xi}_i^d)}}{\partial\bm{\xi}^{\star}_i(\Theta)}$} based on definition \eqref{eq_lost}.\\
  Solving the forward problem of the dynamic game to obtain ${\bm{\xi}^{\star}_i(\Theta^k)}$. (cf. Algorithm \ref{Algorithm_GAME_shooting}, Appendix.)\\	
  Solving a diff-KKT condition to obtain {\small$\displaystyle{\frac{\partial\bm{\xi}^{\star}_i(\Theta)}{\partial\theta_i}}$}. (cf. Algorithm \ref{Algorithm_DPMPsolver}.)\\
  Compute {\small${\displaystyle\left.\frac{d{\mathcal{L}_i({\bm{\xi}^{\star}_i(\Theta)},\bm{\xi}_i^d)}}{d\theta_i}\right|_{\theta_i^k}}$} using \eqref{eq_gradchain}.\\
  Update: {\small$\displaystyle\theta_i^{k+1} = \theta_i^k-\eta^k\left.\frac{d{\mathcal{L}_i({\bm{\xi}^{\star}_i(\Theta)},\bm{\xi}_i^d)}}{d\theta_i}\right|_{\theta_i^k}$}.
	}
	\textbf{Output} $\theta_i$	
\end{algorithm2e}

For the first term of the chain rule, the derivative $\frac{\partial \mathcal{L}_i}{\partial\bm{\xi}^{\star}_i(\Theta)}$ is readily accessible because the function ${\mathcal{L}_i({\bm{\xi}^{\star}_i(\Theta)},\bm{\xi}_i^d)}$ is explicitly defined. Its evaluation point $\bm{\xi}^{\star}_i(\Theta^k)$ relies on solving the forward problem of the game to obtain its \eqref{eq_NashP} with current parameter $\Theta^k$. In this paper, we achieve this by employing an existing distributed Nash equilibrium-seeking algorithm proposed in~\cite{tatarenko2020geometric}. Since this is not the main contribution of the paper, we provide its implementation and justification in the Appendix.  

The major obstacle arises from the second term of the chain rule, where ${\nicefrac{\partial\bm{\xi}^{\star}_i(\Theta)}{\partial\theta_i}}$ characterizes the change in the robot's \eqref{eq_NashP} trajectories corresponding to the change from its local parameter. Given a general optimal control system, its solution trajectory $\bm{\xi}^{\star}_i(\Theta)$ does not admit an analytical form. Thus, one possible way to compute ${\nicefrac{\partial\bm{\xi}^{\star}_i(\Theta)}{\partial\theta_i}}$ is by numerical approximation~\cite{PJP-BK-EM:96}.
However, the feasibility of this approach is extremely challenging, due to the large number of robots and the complexity of their trajectories considered in this paper. Motivated by these, we next present a new \textit{distributed} method to compute ${\nicefrac{\partial\bm{\xi}^{\star}_i(\Theta)}{\partial\theta_i}}$, whose idea is based on differentiating the KKT condition~\cite{BD:12} of the \eqref{eq_NashP} with respect to the parameter $\Theta$~\cite{jin2020pontryagin}. This yields a new representation of the derivative that can significantly reduce its computation burden, and the computation can be performed in a \textit{distributed} fashion.




\subsection{A Fully Distributed Solver for Diff-KKT}
In this subsection, we introduce a distributed and efficient approach to compute the ${\nicefrac{\partial\bm{\xi}^{\star}_i(\Theta)}{\partial\theta_i}}$ in \eqref{eq_gradchain}. 
First, given $x_i^0$, define a compact form for robot $i$'s dynamics constraints
\begin{align}
    \mathbf{F}_i(\mathbf{x}_i,\mathbf{u}_i,\theta_i)=\begin{bmatrix}
        \quad x_i^1-f_i(x_i^0, u_i^0, \theta_i)\\
        \quad x_i^2-f_i(x_i^1, u_i^1, \theta_i)\\
        \vdots\\
        \quad x_i^T-f_i(x_i^{T-1}, u_i^{T-1}, \theta_i)
    \end{bmatrix}=\bm{0}.
\end{align}
The \eqref{eq_NashP} of a game is the collection of the optimal trajectories of the robots' local optimal control problems. Thus, define augmented functions 
\begin{align}\label{eq_def_hi}
    \mathbf{H}_i= \mathcal{J}_i(\mathbf{x}_i,\mathbf{x}_{\mathcal{N}_i},\mathbf{u}_i,\theta_i)+\bm{\lambda}_i^{\top}\mathbf{F}_i(\mathbf{x}_i,\mathbf{u}_i,\theta_i),
\end{align}
with $\bm{\lambda}_i=\{\lambda_1,\cdots,\lambda_m\}$ being the co-states of the dynamics constraints. Then for any $\Theta$, the trajectory $\bm{\xi}^{\star}_{i}(\Theta)=\{{\mathbf{x}^{\star}_{i}}(\Theta),{\mathbf{u}^{\star}_{i}}(\Theta)\}$ must satisfy a distributed discrete-time KKT~\cite{boyd2004convex} condition, which  reads: $\forall i\in\{1,\cdots,m\}$,
\begingroup
\addtolength{\jot}{0.1em}
\begin{subequations}\label{eq_kkt}
\begin{align}\label{eq_kkt_a}
    \frac{\partial \mathbf{H}_i}{\partial \mathbf{x}_i}=\frac{\partial \mathcal{J}_i}{\partial \mathbf{x}_i}+\bm{\lambda}_i^{\top}\frac{\partial \mathbf{F}_i}{\partial \mathbf{x}_i}=\bm{0}\\\label{eq_kkt_b}
    \frac{\partial \mathbf{H}_i}{\partial \mathbf{u}_i}=\frac{\partial \mathcal{J}_i}{\partial \mathbf{u}_i}+\bm{\lambda}_i^{\top}\frac{\partial \mathbf{F}_i}{\partial \mathbf{u}_i}=\bm{0}\\\label{eq_kkt_c}
    \frac{\partial \mathbf{H}_i}{\partial \bm{\lambda}_i}=~~~~~~~~~~~~~~~\mathbf{F}_i=\bm{0}
\end{align}
\end{subequations}
Now, to obtain the ${\nicefrac{\partial\bm{\xi}^{\star}_i(\Theta)}{\partial\theta_i}}$, our idea is to differentiate equation \eqref{eq_kkt} with respect to $\Theta$. This will provide us with a neat and easy-to-solve equation set that directly takes ${\nicefrac{\partial\bm{\xi}^{\star}_i(\Theta)}{\partial\theta_i}}$ as variables. 
To visualize this, define
\begin{align}\label{eq_dist_DPMP2}
	\bm{X}_{i}=\frac{\partial {\mathbf{x}_i^{\star}}(\Theta)}{\partial\Theta}, ~\bm{U}_{i}=\frac{\partial {\mathbf{u}_i^{\star}}(\Theta)}{\partial\Theta}, ~\bm{\Lambda}_{i}=\frac{\partial {\bm{\lambda}^{\star}_i}(\Theta)}{\partial\Theta}.
\end{align}
Since all variables in \eqref{eq_dist_DPMP2} are functions of $\Theta$, differentiating \eqref{eq_kkt} with respect to $\Theta$ yields the following \textbf{Diff-KKT}:
\begin{subequations}\label{eq_diffkkt}
\begin{align}
    M_i^{\alpha}\bm{X}_i+ N_i^{\alpha}\bm{U}_i    +\sum_{j\in\mathcal{N}_i}Q_{ij}^{\alpha}\bm{X}_j+S_i^{\alpha}\bm{\Lambda}_i+C_i^{\alpha}=\bm{0}\\
    M_i^{\beta}\bm{X}_i+ N_i^{\beta}\bm{U}_i    +\sum_{j\in\mathcal{N}_i}Q_{ij}^{\beta}\bm{X}_j+S_i^{\beta}\bm{\Lambda}_i+C_i^{\beta}=\bm{0}\\
    M_i^{\gamma}\bm{X}_i+ N_i^{\gamma}\bm{U}_i+C_i^{\gamma}=\bm{0}
\end{align}
\end{subequations}
with the application of the chain rule on the derivatives of $\mathbf{x}_i^{\star}(\Theta)$ and $\mathbf{u}_i^{\star}(\Theta)$ and $\bm{\lambda}_i^{\star}(\Theta)$ with respect to $\Theta$:
\begin{subequations}\label{eq_def_MNQ}
{\small\begin{align}
\begin{aligned}
    M_i^{\alpha}=\frac{\partial^2 \mathbf{H}_i}{\partial {\mathbf{x}^\star_i}^2 },~
    N_i^{\alpha}=\frac{\partial^2 \mathbf{H}_i}{\partial \mathbf{x}_i^\star \partial \mathbf{u}_i^\star },~
    Q_{ij}^{\alpha}=\frac{\partial^2 \mathbf{H}_i}{\partial \mathbf{x}_i^\star \partial \mathbf{x}_j^\star } \\
    S_{i}^{\alpha}=\frac{\partial^2 \mathbf{H}_i}{\partial \mathbf{x}_i^\star \partial \bm{\lambda}_i^\star },~
    C_i^{\alpha}=\frac{\partial^2 \mathbf{H}_i}{\partial \mathbf{x}_i^\star \partial \theta_i^k }\end{aligned} \\
    \begin{aligned}
    M_i^{\beta}=\frac{\partial^2 \mathbf{H}_i}{\partial \mathbf{u}_i^\star \partial \mathbf{x}_i^\star },~
    N_i^{\beta}=\frac{\partial^2 \mathbf{H}_i}{ \partial {\mathbf{u}_i^\star}^2 },~
    Q_{ij}^{\beta}=\frac{\partial^2 \mathbf{H}_i}{\partial \mathbf{u}_i^\star \partial \mathbf{x}_j^\star }\\
    S_{i}^{\beta}=\frac{\partial^2 \mathbf{H}_i}{\partial \mathbf{u}_i^\star \partial \bm{\lambda}_i^\star },~
    C_i^{\beta}=\frac{\partial^2 \mathbf{H}_i}{\partial \mathbf{u}_i^\star \partial \theta_i^k }\end{aligned}\\
    M_i^{\gamma}=\frac{\partial^2 \mathbf{H}_i}{\partial \bm{\lambda}_i^\star \partial \mathbf{x}_i^\star },~
    N_i^{\gamma}=\frac{\partial^2 \mathbf{H}_i}{ \partial \bm{\lambda}_i^\star \partial\mathbf{u}_i^\star },~
    C_i^{\gamma}=\frac{\partial^2 \mathbf{H}_i}{\partial \bm{\lambda}_i^\star \partial \theta_i^k }
\end{align}}
\end{subequations}
\endgroup 
where we use $\frac{\partial^2 \mathbf{H}_i}{\partial {\sigma}_i^\star \partial {\mu}_i^\star }$ to denote the seconder-order derivative of $\mathbf{H}_i(\cdot)$ evaluated at $\{{\sigma^{\star}_i}(\Theta), {\mu^{\star}_i} (\Theta)\}$. All equations in \eqref{eq_def_MNQ} are simple numerical matrices and are readily computable from \eqref{eq_kkt}, because $\mathbf{H}_i(\cdot)$ is explicitly defined and  $\{\mathbf{x}_i^\star(\Theta), \mathbf{u}_i^\star(\Theta)\}, \bm{\lambda}_i^\star(\Theta)$ are obtained from forward Nash seeking algorithm (described in the Appendix) given the current $\Theta$.  
To remark the effectiveness of reformulation, given Assumption \ref{Ass_func}, results in \cite[Sec. 5.9.2]{boyd2004convex} implies the existence and uniqueness of solution to \eqref{eq_kkt}; results in~\cite[Theorem 1]{jin2021safe} implies the uniqueness of $\mathbf{X}_i$ and $\mathbf{U}_i$ in \eqref{eq_diffkkt}.

\begin{algorithm2e}[t]
	\setstretch{1.0}
	\caption{Distributed Solver for Diff-KKT, the local update for robot $i$.}
	\label{Algorithm_DPMPsolver}
	\SetAlgoLined
	\textbf{Input}  $\bm{\xi}_i^{\star}(\Theta^k)$, $\theta_i^k$.\\
	\textit{Compute} ${\bm{\lambda}^{\star}_{i}}(\Theta^k)$ using equations \eqref{eq_kkt} with $\bm{\xi}^{\star}_{i}(\Theta)=\{{\mathbf{x}^{\star}_{i}}(\Theta^k),{\mathbf{u}^{\star}_{i}}(\Theta^k)\}$.\\
	\textit{Compute} matrices $\bm{A}_{i,i}$ and $\bm{A}_{i,j}$, $j\in\mathcal{N}_i$ by \eqref{eq_def_MNQ} and \eqref{eq_defYAA}.\\
	
	\textit{Acquire} matrices $\bm{A}_{\ell,i}$, $\ell\in\mathcal{N}_i$ from each neighbor $\ell$ of robot $i$. Assign $\bm{A}_{\ell,i}=\bm{0}$ for $\ell\notin \mathcal{N}_i$.\\
	
	\textit{Compose} matrices $\bm{\Psi}_i$, $\widehat{\bm{C}}_i$ by their definitions.
	
	\textit{Initialize} $\tau=0$, $\delta\in\mathbb{R}_{+}$, and $\bm{Y}_i^{\tau=0}$, $\bm{Z}_i^{\tau=0}$ as random matrices with proper sizes.
	
	\While{$\max_i(|\bm{Y}_i^{\tau+1}-\bm{Y}_i^{\tau}|)\ge\epsilon_{\bm{Y}}$ \label{st_start}}
	{
   Exchange states $\bm{Z}_i^{\tau}$ among neighboring robots.\\
   State update:
   \begin{align*}
       &\bm{v}_i^\tau=\bm{\Psi}_i\bm{Y}_i^\tau-\widehat{\bm{C}}_i-\sum_{\ell\in\mathcal{N}_i}(\bm{Z}_{i}^\tau-\bm{Z}_\ell^\tau).\\
		&\bm{Y}_i^{\tau+1}=\bm{Y}_i^\tau-\delta\bm{\Psi}_i^{\top}\bm{v}_i^\tau\\  
		&{\bm{Z}}_i^{\tau+1}={\bm{Z}}_i^\tau+\delta\bm{v}_i^\tau
   \end{align*}
	}\label{st_end}
	\textit{Obtain} $\bm{X}_i$, $\bm{U}_i$ from $\bm{Y}_i^\tau$ based on \eqref{eq_defYAA}.\\
	\textbf{Output} $\displaystyle\left.{\frac{\partial\bm{\xi}^{\star}_i(\Theta)}{\partial\theta_i}}\right|_{\theta_i^k}$ from $\left\{\bm{X}_i,\bm{U}_i\right\}$ based on \eqref{eq_dist_DPMP2}.
\end{algorithm2e}

\noindent{\textbf{Distributed Diff-KKT Solver:}}
Solving (\ref{eq_dist_DPMP2}) from \eqref{eq_diffkkt} gives us the  gradient ${\frac{\partial\bm{\xi}^{\star}_i(\Theta)}{\partial\theta_i}}$ for each robot. However, solving the equation in a centralized manner is not scalable as the robot number grows.
To address this, we notice that the coupled terms, i.e., $Q_{ij}, \bm{X}_j$, in (\ref{eq_diffkkt}) only exist among connected neighbors $j\in\mathcal{N}_i$. This motivates us to develop a fully distributed solver to compute the gradient. To that end, we rewrite all variables and matrices into a compact linear equation form.
\begin{align}\label{eq_dist_DPMPLEQ2}
		\bm{A}_{i,i} \bm{Y}_i+\sum_{j\in\mathcal{N}_i}({\bm{A}}_{i,j}\bm{Y}_j)+\overline{C}_i=\bm{0}.
\end{align}
where for all $i$ and $j\in\mathcal{N}_i$,
\begingroup
\renewcommand*{\arraystretch}{1.2}
\begin{equation}
\begin{aligned}\label{eq_defYAA}
		\bm{A}_{i,i}=\begin{bmatrix}
			{M}_i^\alpha~~{N}_i^\alpha~~{S}_i^\alpha\\
   {M}_i^\beta~~{N}_i^\beta~~{S}_i^\beta\\
   {M}_i^\gamma~~{N}_i^\gamma~~{S}_i^\gamma
		\end{bmatrix},~~~
	\bm{Y}_i=\begin{bNiceMatrix}[l]
	\bm{X}_i\\\bm{U}_i\\\bm{\Lambda}_i
\end{bNiceMatrix}
\\[2pt]
	\bm{A}_{i,j}=\begin{bmatrix}
		{Q}_{i,j}^\alpha~~\bm{0}~~\bm{0}\\
  {Q}_{i,j}^\beta~~\bm{0}~~\bm{0}\\
  {Q}_{i,j}^\gamma~~\bm{0}~~\bm{0}
	\end{bmatrix},~~~
 \overline{C}_i=\begin{bmatrix}
		C_i^\alpha\\
  C_i^\beta\\
  C_i^\gamma
	\end{bmatrix}
\end{aligned}
\end{equation}
\endgroup
where $\bm{Y}_i$ is the local unknown of robot $i$, $\bm{A}_{i,i}$ and $\bm{A}_{i,j}$ are known matrices, and $\bm{Y}_j$, $j\in\mathcal{N}_i$ is the coupled unknown from $i$'s neighbors. Since each robot in the network possesses an equation in the form of \eqref{eq_dist_DPMPLEQ2}, to compute a set of $\bm{Y}_i$, $i\in\until{m}$ satisfying all these equation, we essentially need to solve the following  compact equation set
\begin{align}\label{eq_GLE}
	\sum_{i=1}^{m}(\bm{\Psi}_i\bm{Y}_i+\widehat{\bm{C}}_i)=0 
\end{align}
where $\bm{\Psi}_i =\begin{bmatrix}
\bm{A}_{1,i}^{\top}, ...,\bm{A}_{m,i}^{\top}
	\end{bmatrix}^{\top}$, $\widehat{\bm{C}}_i =\begin{bmatrix}
			\bm{0},...,\overline{C}_{i}^{\top},...,\bm{0}
		\end{bmatrix}^{\top}$. The matrix is a zero matrix if undefined.
In $\widehat{\bm{C}}_i$, the matrix $\overline{C}_{i}^{\top}$ is located at the $i$th block. By stacking the matrices $\bm{A}_{\ell,i}$ and $\overline{C}_{i}^{\top}$, each row block of \eqref{eq_GLE} is associated with one  \eqref{eq_dist_DPMPLEQ2} for $i\in\until{m}$. Further note that network $\mathbb{G}$ is undirected, i.e., $i\in\mathcal{N}_\ell$ yields $\ell\in\mathcal{N}_i$, thus, robot $i$ has access to $\bm{\Psi}_i$ based on its local communication with its neighbors.
Now, suppose each robot $i$ knows $\bm{\Psi}_i$ and $\widehat{\bm{C}}_i$, we introduce Algorithm \ref{Algorithm_DPMPsolver} for the robots to efficiently solve its $\bm{Y}_i$.

Algorithm \ref{Algorithm_DPMPsolver} is fully distributed, in the sense that the computation of each robot only relies on its own state and the states of its neighbors. It leverages our preliminary result in~\cite{XSB19TAC}. The convergence of the algorithm is characterized by the following result with its proof in the Appendix.
\begin{lemma}\longthmtitle{Validity of Algorithm \ref{Algorithm_DPMPsolver}} \label{LM_alg3}
	Suppose the network $\mathbb{G}$ is undetected and connected, suppose equation set \eqref{eq_GLE} has a unique solution, by Algorithm \ref{Algorithm_DPMPsolver}, if the positive step-size $\delta$ is sufficiently small, the state $\bm{Y}_i^{\tau}$ of robot $i$ will converge asymptotically to a state $\bm{Y}_i^{\star}$, where the set of $\{\bm{Y}_i^{\star},i=1\cdots,m\}$ forms a solution to \eqref{eq_GLE}.
\end{lemma}
\section{Experiments} \label{Sec_sim}
This section presents simulation experiments to validate the effectiveness, scalability, and generalizability of the proposed D3G approach for multi-robot coordination. The system includes two types of robots: TurtleBot3 Burger and Waffle in Fig. \ref{Fig_robmodel}. We consider heterogeneous settings, where each robot has different dynamics, such as different radii, weights, and velocity/angular ranges. 
Four scenarios are used: (a) fixed swapping in open ground, (b) formation initialization using the environment in the introductory Fig. \ref{ALG_demo}, (c) cooperative payload transportation, and (d) formation maintenance using the environments in Fig. \ref{Fig_GazeboEnv}.
Simulations are done in Gazebo via ROS. Robots can communicate with each other, but all computations are performed locally. 


\begin{figure}[h]
	\vspace{-1ex}
	\centering
	\begin{subfigure}[h]{0.12\textwidth}
		\centering
		\includegraphics[width=\textwidth]{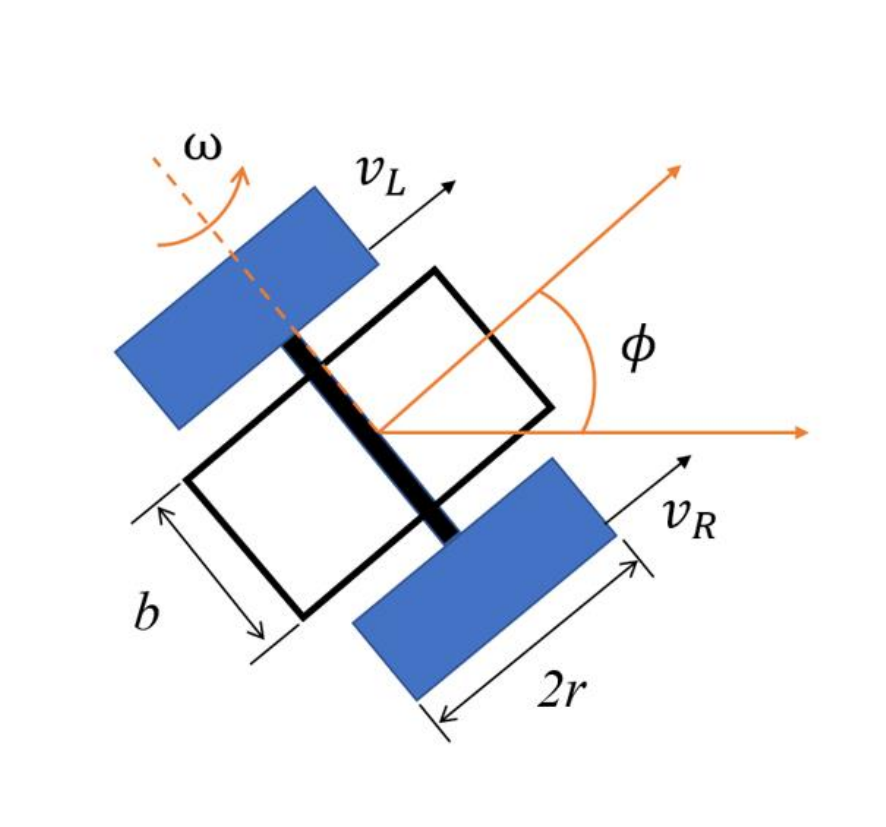}
		\caption{Differential- drive model. }
		\label{Fig_diffd}
	\end{subfigure}
	\begin{subfigure}[h]{0.12\textwidth}
		\centering
		\includegraphics[width=\textwidth]{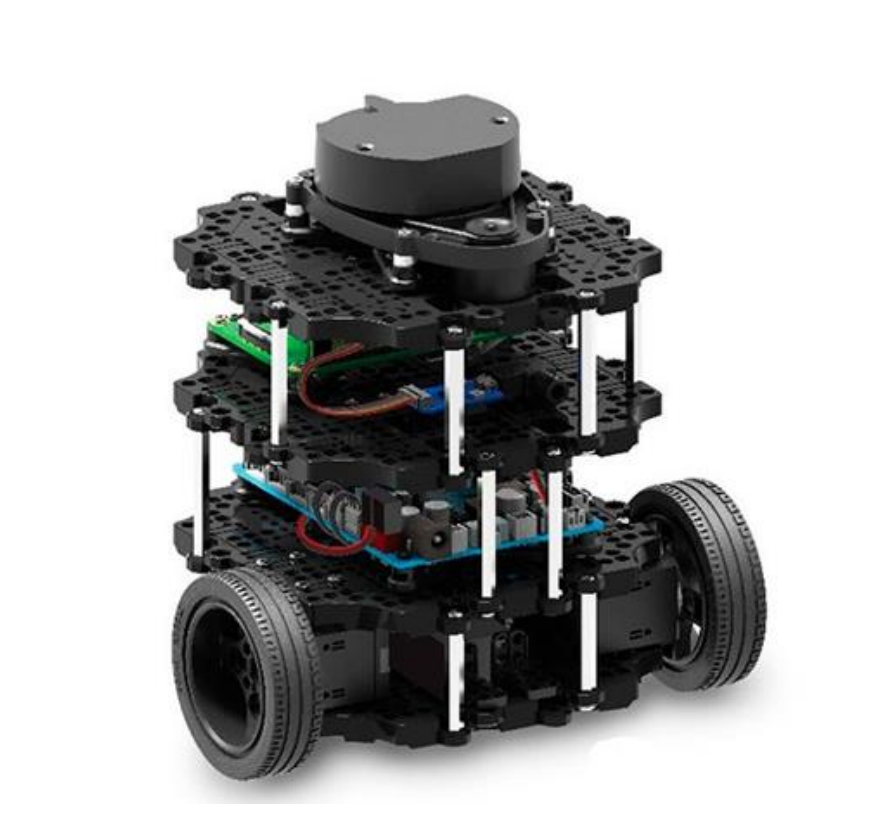}
		\caption{Turtlebot3 burger.}
		\label{Fig_tb3_burger}
	\end{subfigure}
	\begin{subfigure}[h]{0.12\textwidth}
		\centering
		\includegraphics[width=\textwidth]{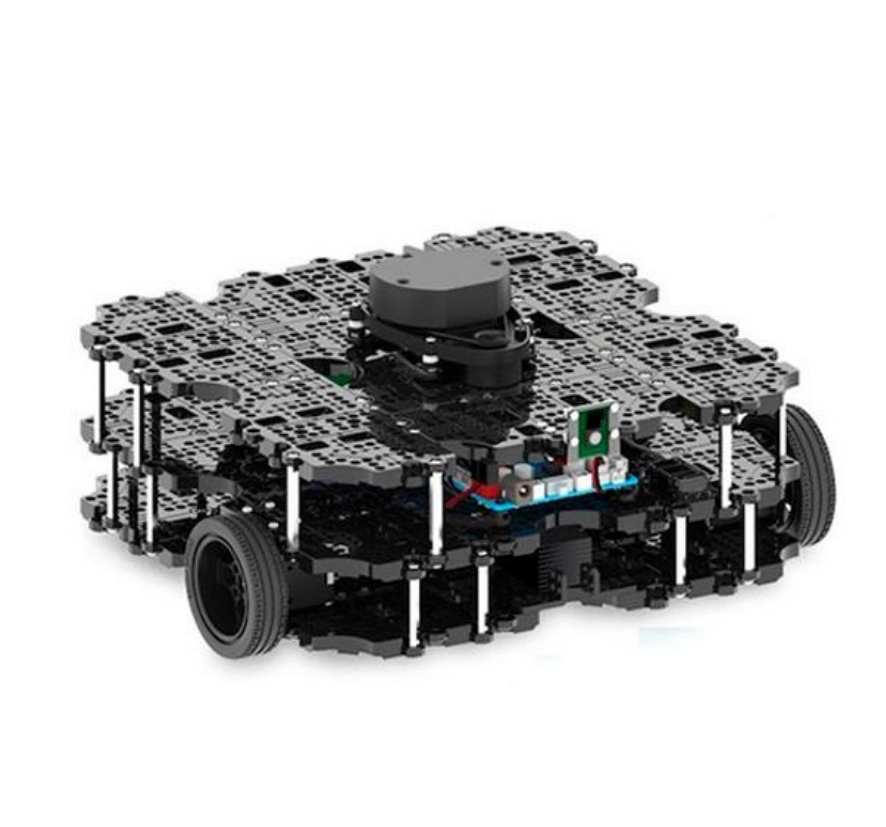}
		\caption{Turtlebot3 waffle.}
		\label{Fig_tb3_waffle}
	\end{subfigure}
\vspace{-.5ex}	
 \caption{Experimental Platform: Turtlebot3 and its model.\vspace{-1.5ex}}
	\label{Fig_robmodel}
\end{figure}

\begin{figure}[t]
	\vspace{-1ex}
	\centering
	\includegraphics[width=0.45\textwidth]{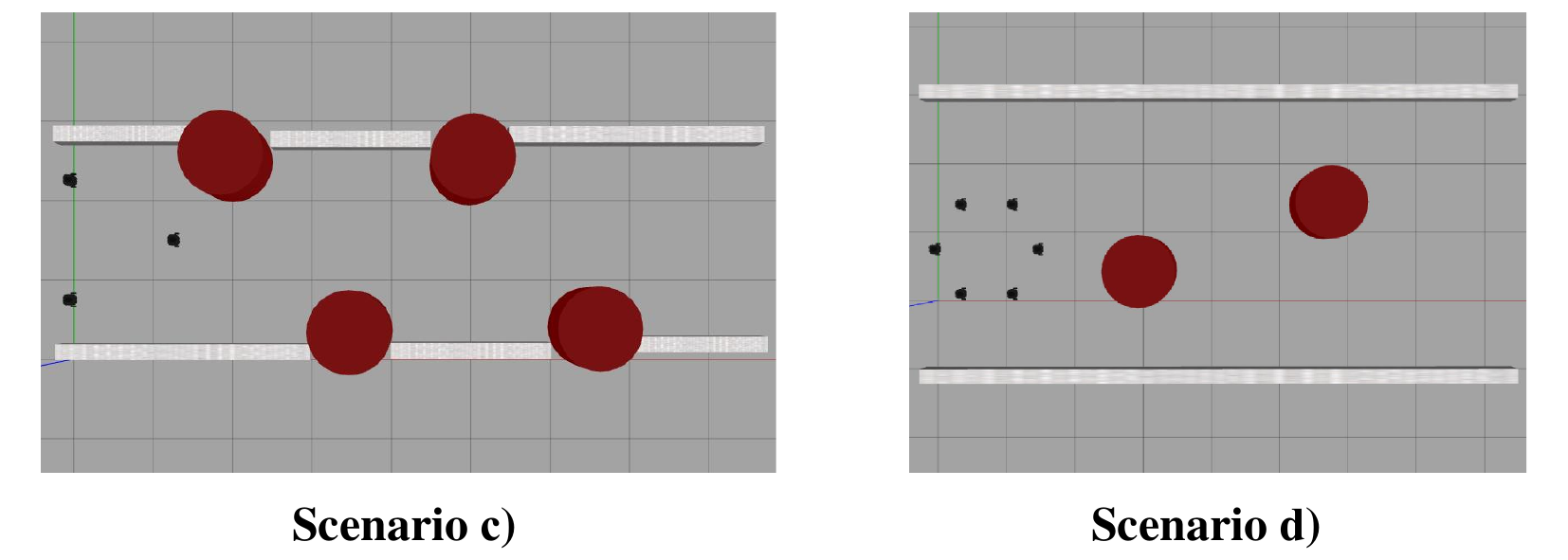}
\vspace{-.5ex}	
 \caption{Gazebo environment for scenarios c) and d).}
	\label{Fig_GazeboEnv}
	\vspace{-.5ex}
\end{figure}



\noindent\textbf{Parameterization of objective functions:} Fcuntion $\mathcal{J}_i$ is parameterized by considering a linear combination of the following cost terms with unknown weights: \textit{formation maintenance}, which defines the positional relationship of neighboring robots in terms of their relative positions, distances, or velocities; \textit{risk/obstacle avoidance}, which employs a reciprocal function to repel robots from given risk areas; \textit{collision avoidance}, which utilizes a reciprocal function to prevent robots from colliding with each other; and \textit{waypoint following}, which provides sparse navigation cues for navigating complex environments. We note that these functions satisfy Assumption~\ref{Ass_func}.


\noindent\textbf{Experiment Settings in Each Scenario}: 
We invite humans to create several sets of trajectories (incorporating human-induced random noise to optimal coordination trajectories computed from N.E. of a game with parameter $\Theta^{\star}$) to serve as the expert demonstration data. Using Algorithm \ref{Algorithm_DIDG}, we learn $\theta_i^{\star}$ for each robot from those demonstrations. Additionally, for each scenario, we test the generalizability of the learned objective functions by applying them in a new environment where the robots can still generate appropriate coordinated behaviors.
Details of simulation setups and results are as follows:

\noindent \textbf{Scenario a)}: We solve a multi-robot fixed swapping task. As shown in Fig.~\ref{Fig_Scenario4}-a, in the demonstrations, six robots are initialized around a circle-like formation. 
Each robot navigates to the diagonally opposite goal position on the other side of the circle. Throughout the process, they must dynamically adjust their positions to move without colliding. We test the generalization of the learned objective function with an increased number of robots, and the task is accomplished very well. Fig.\ref{Fig_Scenario4}-b shows an example with \textbf{sixteen} robots.
\begin{figure}[H]
	\vspace{-1em}
	\centering
	\begin{subfigure}[h]{0.45\textwidth}
            \centering
		\includegraphics[width=\textwidth]{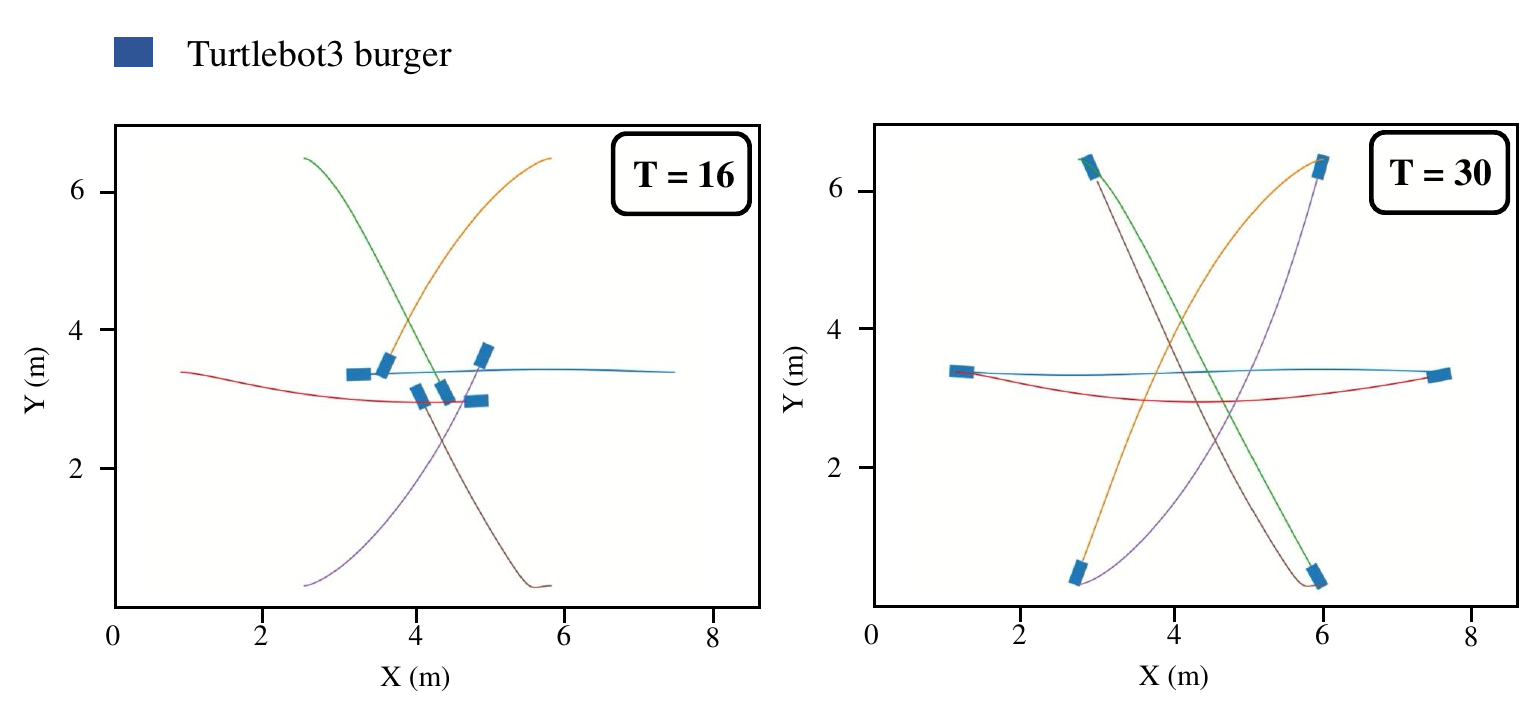}
		\caption{\small Inverse game solving: learning the fixed swapping task from demonstrations. (Right: the reproduced motion)}
		\label{Fig_s0_learn}
	\end{subfigure}
	\begin{subfigure}[h]{0.45\textwidth}
		\centering
		\includegraphics[width=\textwidth]{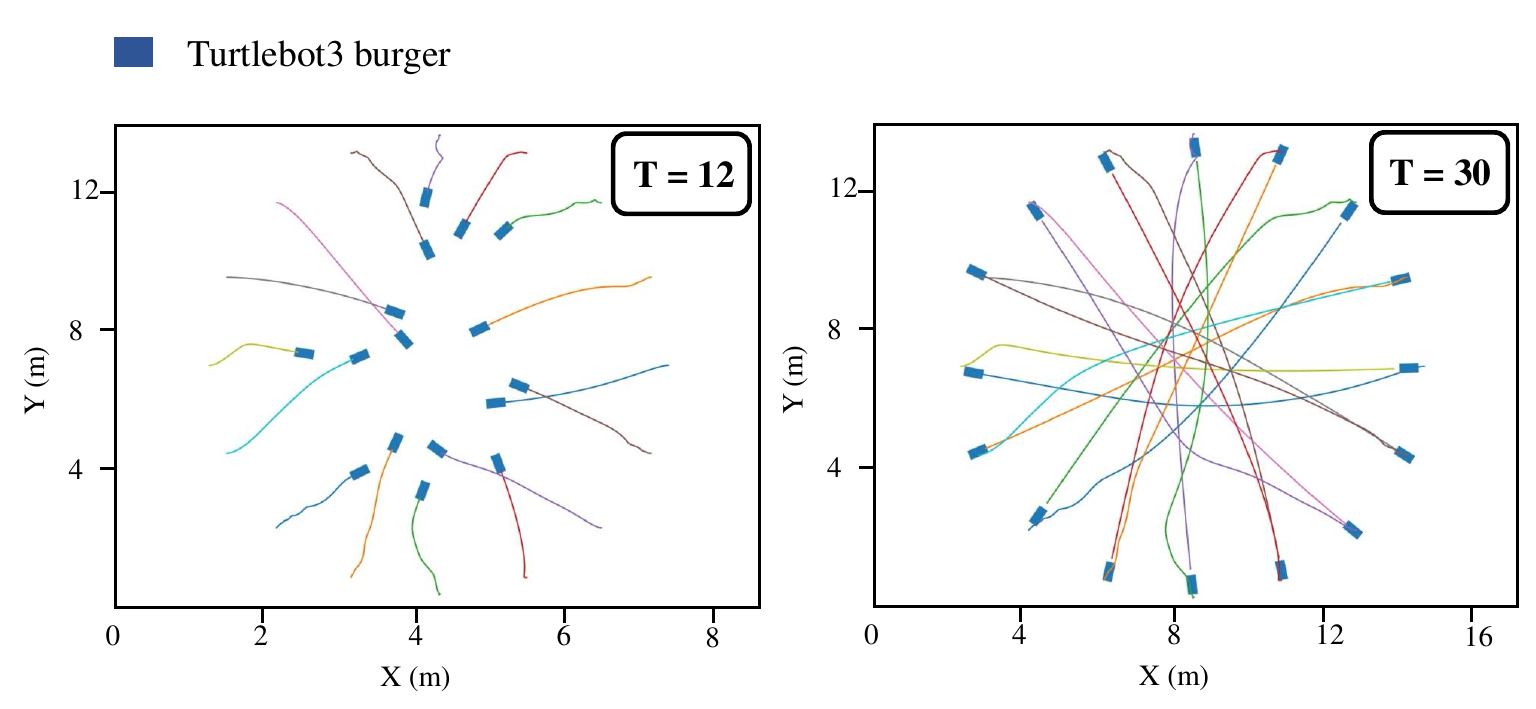   }
		\caption{\small Generalization of the learned objective with 16 robots.}
		\label{Fig_s0_valid}
	\end{subfigure}
 \vspace{-.5ex}
	\caption{Learning fixed swapping tasks with sixteen robots.}
    \label{Fig_Scenario4}
\end{figure}

\noindent\textbf{Scalability of Distributed Solver}: 
Using different numbers of robots in scenario (a), we compare the computational scalability of the proposed algorithm with the GT-IRL~\cite{CK-XL:22} and IKKT~\cite{EP-VNA-TM:17} methods. The comparison result is presented in Fig.\ref{Fig_time}. 
Here, D3G is evaluated based on the per-iteration time of Algorithm 1, which requires the convergence of Algorithm 2 for the inverse pass and Algorithm 3 for the forward pass. 
Since both algorithms are gradient-based and are sensitive to initial values, we use the result of the last iteration in Algorithm 1 as the initial values for the new iteration. 
The stopping criteria are chosen such that the variables do not change $1\%$ of their initial values (around hundreds of iterations).
For GT-IRL, its forward pass employs a similar but centralized gradient-based method to solve a dynamic game, and the inverse pass uses a centralized linear equation solver. The IKKT method uses a constraint optimization formulation, which is solved iteratively without a forward/inverse structure. From Fig.~\ref{Fig_time} and the trend of the data, we observe that as the number of robots increases, D3G outperforms both GT-IRL and IKKT in terms of computation time. The inverse pass of D3G outperforms GT-IRL. For D3G, the local computation of each robot is not significantly affected by the system size as the others, thanks to the distributed nature of the algorithm. The increase in time is mainly because Algorithms 2 and 3 require more iterations to converge. In contrast, for centralized algorithms, the computation time grows quickly due to the increase in the number of variables and constraints.

\begin{figure}[h]
	\vspace{-1ex}
	\centering
	\includegraphics[width=0.46\textwidth]{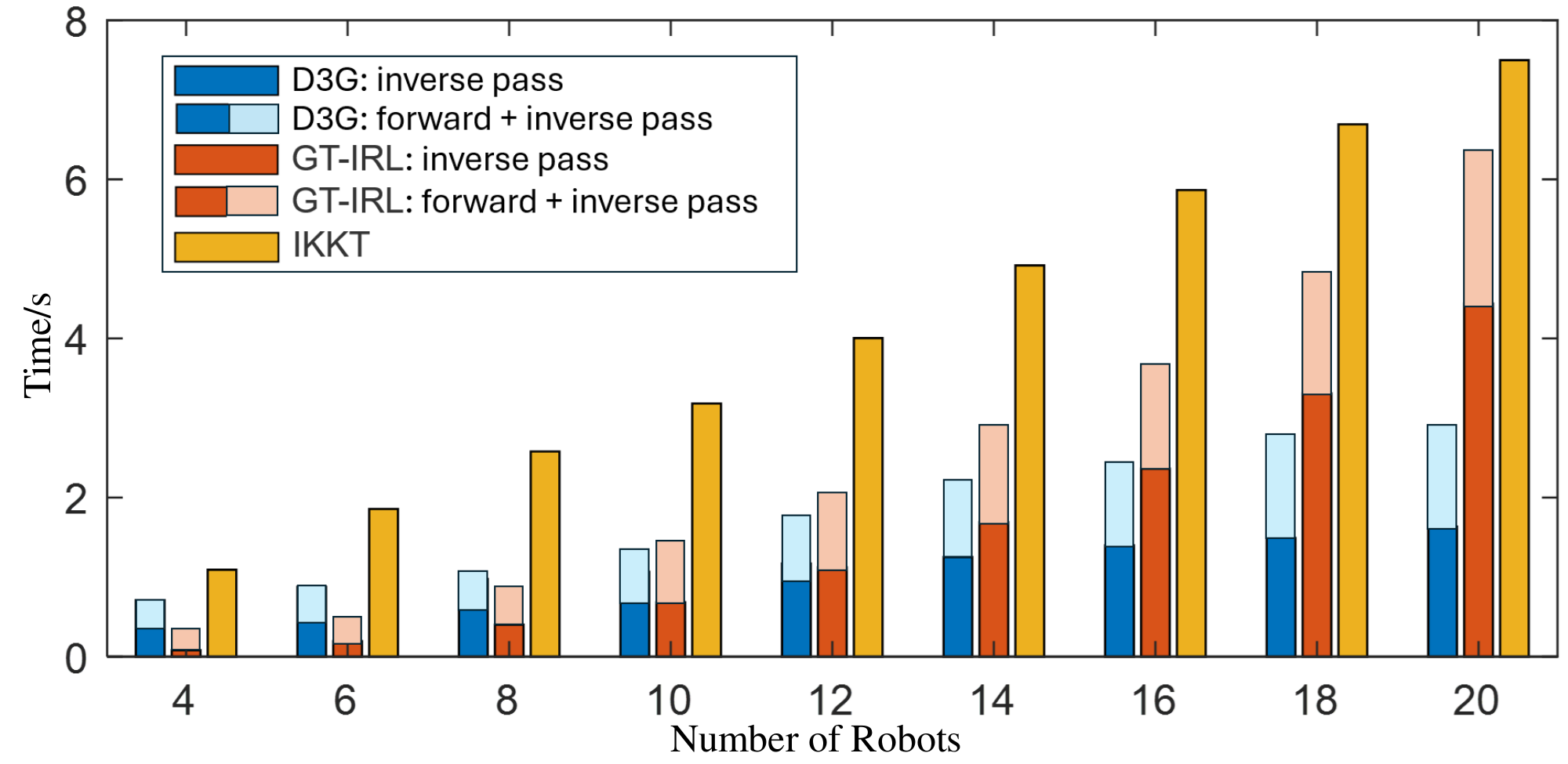}
	\caption{\small Comparison of computation time with GT-IRL and IKKT.
        \label{Fig_time}
 }
\end{figure}

\noindent\textbf{Scenario b)}: 
As shown in the introductory example in Fig. \ref{ALG_demo}, three robots start from initial positions at $0$ speed to initialize a linear formation at the goal position, maintaining distances of $0.8$m and velocities of $0.2$m/s. There exists a wide obstacle that robots have to avoid. 
From the demonstrations, the robots learn to adjust their formation to a `compact' shape when moving through the narrow space, then recover and form the desired formation at target positions. 
To test the generalization of the learned objective functions, we solve the learned game but change the obstacle's opening position from the middle to the side. The robots can still generate proper coordination to initialize the formation.

\noindent \textbf{Scenario c)}: As shown in Fig.~\ref{Fig_Scenario2}, three turtlebots start from different initial positions and cooperatively transport a slung payload. We assume each robot is attached to the payload with a length tether visualized in Fig.~\ref{Fig_s2_learn}. The payload has to maintain clearance from the ground. In addition, to stabilize the payload and prevent excessively large forces between the robots and the payload, the robot team will learn to maintain an equilateral triangle-like form, and keep the payload in its centroid. For simplicity, we ignore the dynamics of the payload but only consider the equilibrium point as its location. 
By learning robots' local objective functions, the reconstructed trajectories are shown in the right plot of Fig.~\ref{Fig_s2_learn}. We then test the generalization of the model in a new environment. In Fig.~\ref{Fig_s2_valid}, the placement of obstacles requires more sophisticated robot maneuvers. The height of the payload is still well maintained, and the robot team keeps the payload in its centroid as much as possible for stable moving.



\begin{figure}[ht]
	\vspace{-1ex}
	\centering
	\begin{subfigure}[h]{0.47\textwidth}
            \centering
		\includegraphics[width=\textwidth]{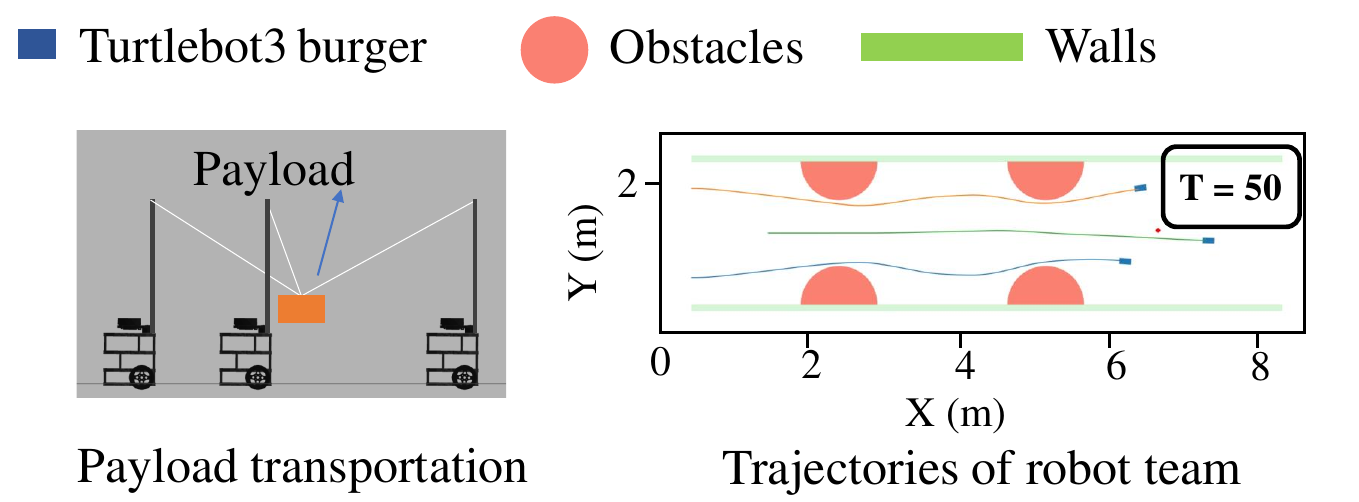}
		\caption{\small Inverse game solving: learning payload transportation from demonstrations. (Right: the reproduced motion at convergence)}
		\label{Fig_s2_learn}
	\end{subfigure}
	\begin{subfigure}[h]{0.47\textwidth}
		\centering
		\includegraphics[width=\textwidth]{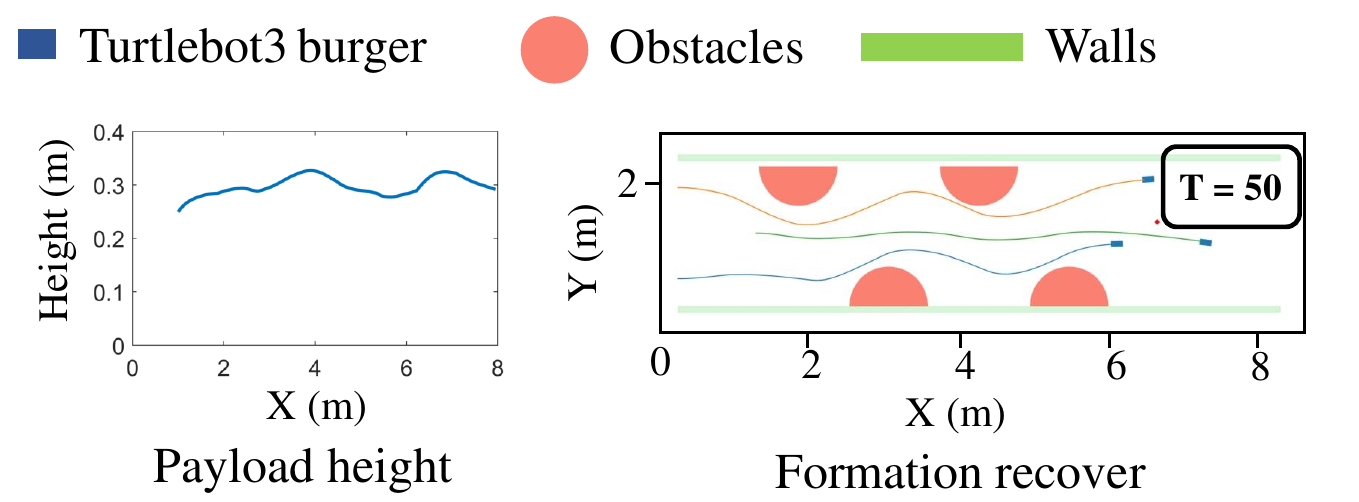}
		\caption{\small Generalization of the learned objective in a new environment.}
		\label{Fig_s2_valid}
	\end{subfigure}
 \vspace{-.5ex}
	\caption{Learning payload transportation with three robots.}
    \label{Fig_Scenario2}
\end{figure}



\noindent \textbf{Scenario d)}: As shown in Fig.~\ref{Fig_Scenario3}, six heterogeneous turtlebot3 robots, including three burgers and three waffles, maintain a desired (circle-like) formation while navigating through complex environments with obstacles. Robots learn to balance between local objective functions including collision avoidance and formation maintenance. 
The reconstructed trajectories in Fig.~\ref{Fig_s3_learn} show the robot's capability to leverage the shape of the obstacle to minimize the formation degradation. We test the generalization of the learned game in \ref{Fig_s3_valid} in a new environment with \textbf{eight} robots. The robots generate smooth trajectories and formation transitions. Furthermore, we observe two robots change their orders (T=0: different types of robots are separated v.s. T=70: two blue/red robots become adjacent) to reduce the formation degradation.
\begin{figure}[ht]
	\vspace{-2ex}
	\centering
	\begin{subfigure}[h]{0.46\textwidth}
		\centering
		\includegraphics[width=\textwidth]{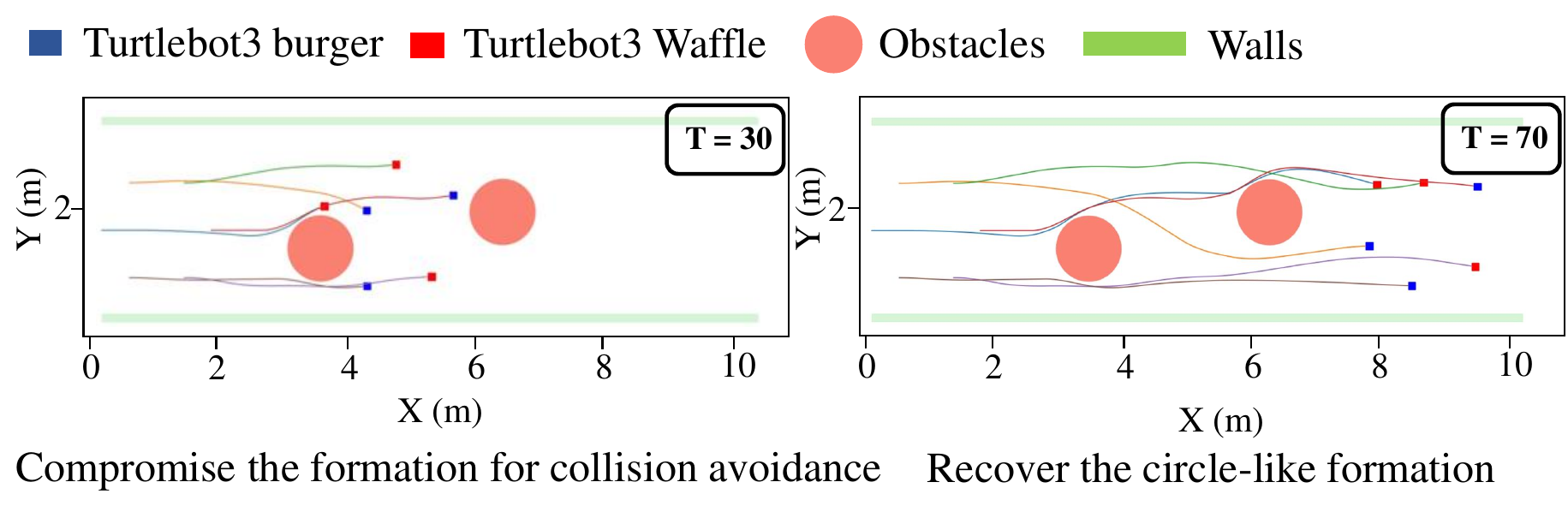}
		\caption{\small Inverse game solving: learning formation maintenance with six heterogeneous robots from demonstrations, figures show reproduced motion at convergence.}
		\label{Fig_s3_learn}
	\end{subfigure}
	\begin{subfigure}[h]{0.46\textwidth}
		\centering
		\includegraphics[width=\textwidth]{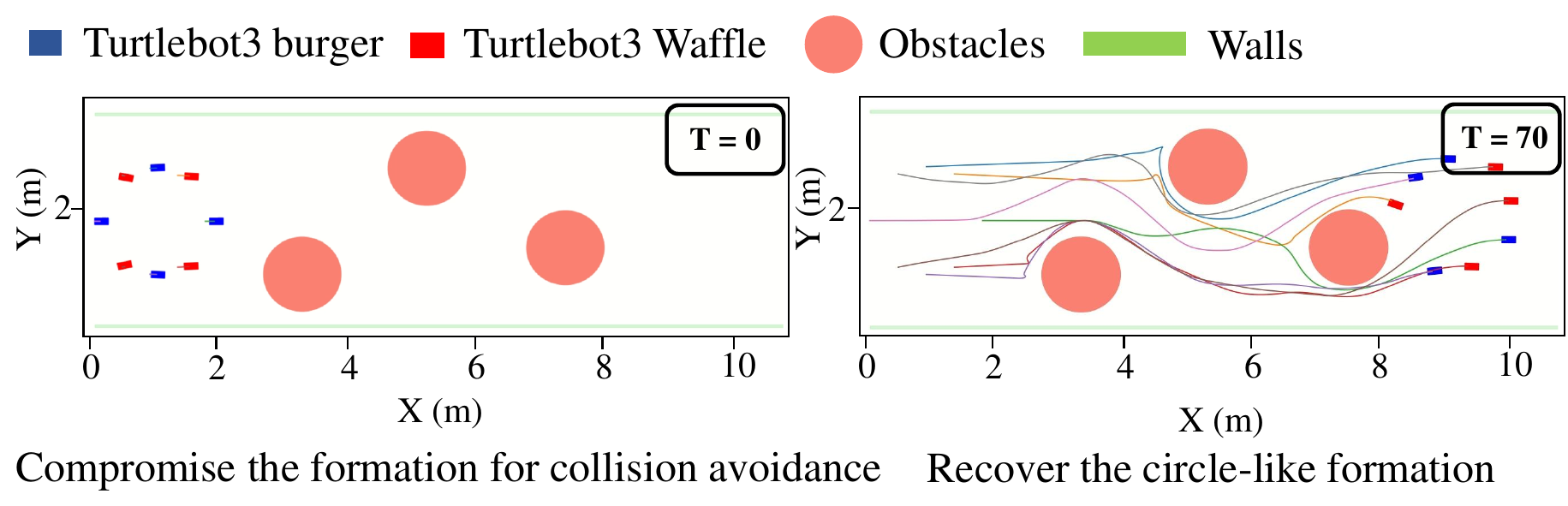}
		\caption{\small Generalization of the learned objective with eight heterogeneous robots in a new environment.}
		\label{Fig_s3_valid}
	\end{subfigure}
 \vspace{-.5ex}
	\caption{Learning formation control with heterogeneous robots.  \vspace{-1ex}}
    \label{Fig_Scenario3}
\end{figure}

\noindent\textbf{Comparison of Learning loss:}
We compare the convergence of the proposed method with the centralized IKKT method~\cite{EP-VNA-TM:17}. The GT-IRL~\cite{CK-XL:22} is not included since it is also based on the diff-KKT condition, leading to a similar convergence property as D3G in terms of learning loss.
The results of all scenarios are shown in Fig.~\ref{Fig_results_c}, where the y-axis represents the learning loss $\mathcal{L}_i$ for each robot, or the total learning loss for the whole system. 
In all scenarios, the total learning loss converges, and the parameter values will converge to those of the demonstrations.
Apart from the advantage in computation scalability demonstrated previously in Fig. \ref{Fig_time}, the proposed D3G, which is fully distributed, demonstrates a comparable, and in some cases, better convergence speed than the centralized IKKT.

\begin{figure}[ht]
   \vspace{-1ex}
   \centering
   \begin{subfigure}[h]{0.23\textwidth}
		\centering
		\includegraphics[width=\textwidth]{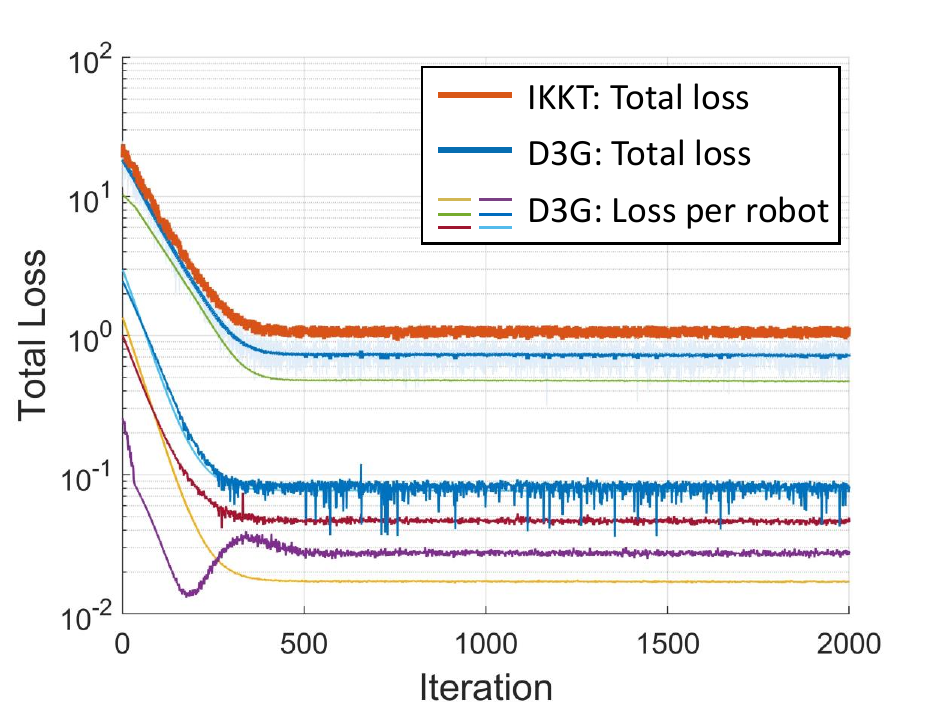}
		\caption{Learning loss of scenario \textbf{a}.}
		\label{Fig_Loss1}
   \end{subfigure}
   \begin{subfigure}[h]{0.23\textwidth}
		\centering
		\includegraphics[width=\textwidth]{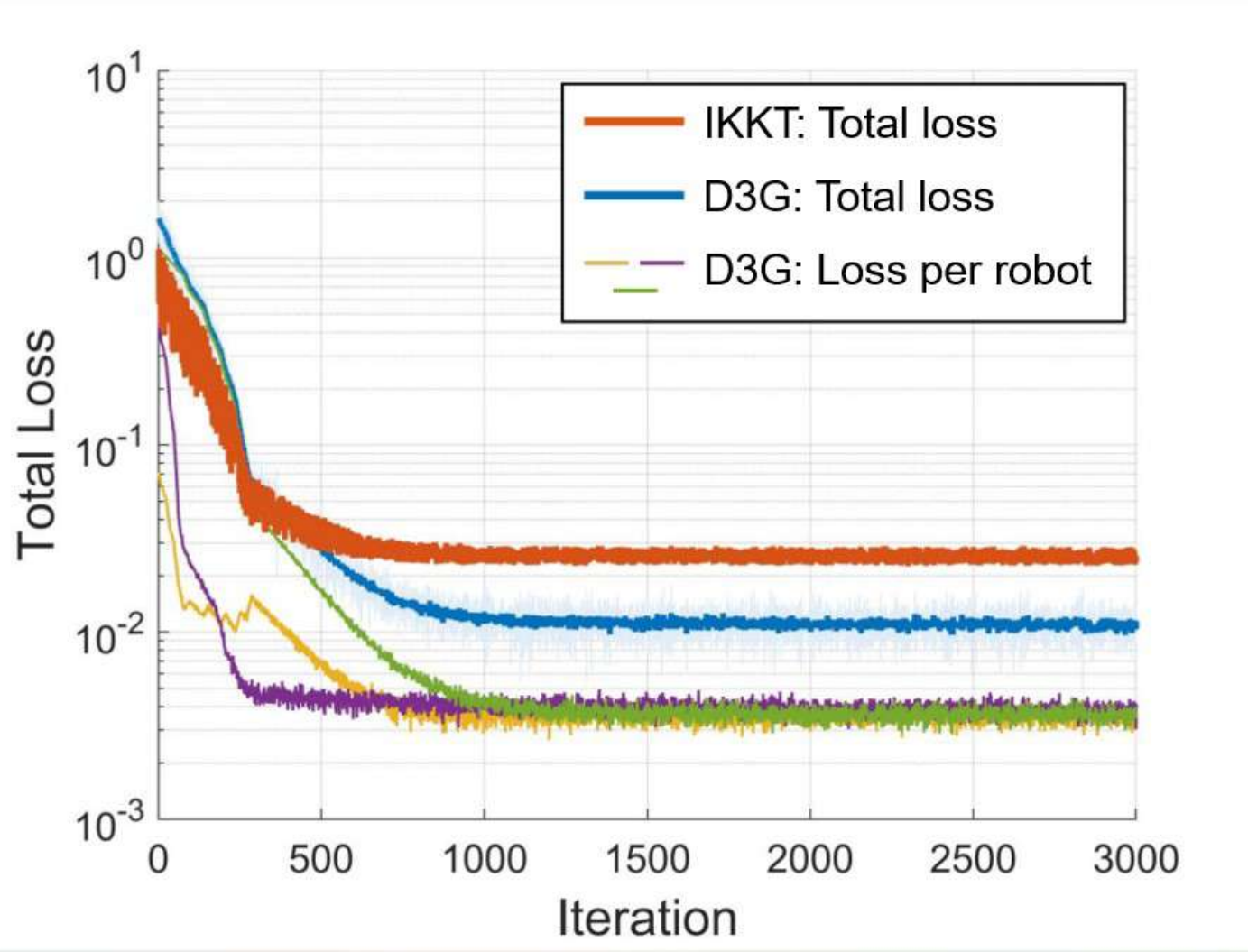}
		\caption{Learning loss of scenario \textbf{b}.}
		\label{Fig_Loss2}
   \end{subfigure}
   \begin{subfigure}[h]{0.23\textwidth}
		\centering
		\includegraphics[width=\textwidth]{Figs/Total_loss2.pdf}
		\caption{Learning loss of scenario \textbf{c}.}
		\label{Fig_Loss3}
   \end{subfigure}
   \begin{subfigure}[h]{0.23\textwidth}
		\centering
		\includegraphics[width=\textwidth]{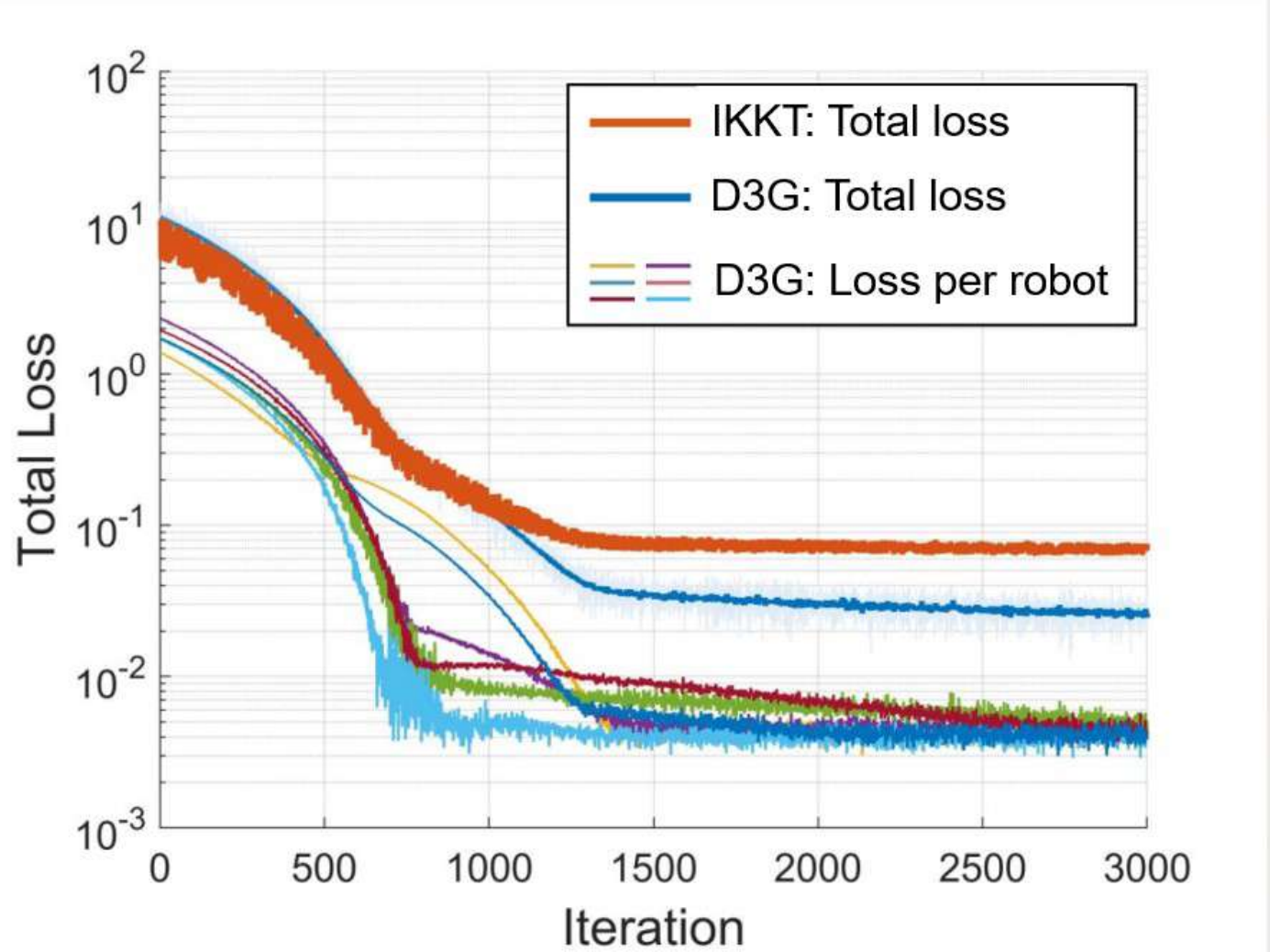}
		\caption{Learning loss of scenario \textbf{d}.}
		\label{Fig_Loss4}
   \end{subfigure}
   \caption{Total loss $\sum_{i=1}^m\mathcal{L}_i$ (with trajectory mismatch defined in \eqref{eq_lost}) of each scenario.}
   \label{Fig_results_c}

\end{figure}

\section{Conclusion and Future Work} \label{Sec_Conclu}
We have developed a new approach for inverse learning of a \emph{Distributed Differentiable Dynamic Game} (D3G), which aims to efficiently learn multi-robot coordination from demonstrations using robots' local information exchange. We represented multi-robot coordination as the Nash equilibrium of a parameterized dynamic game. The goal was to learn the parameters of the game so that it can reconstruct desired multi-robot coordination.
To this end, we developed a distributed inverse dynamic game algorithm with a solver for the diff-KKT condition that allows robots to cooperatively learn parameters for their dynamics and objective functions.   
We have shown the effectiveness of the proposed algorithm through analysis and high-fidelity Gazebo simulations and compared it with existing methods. 
For future works, we will implement the proposed framework into real-robot platforms such as ground and air vehicles for formation control, cooperative transportation, and navigation, where uncertainties, control disturbances, and communication delays will be considered. 
We also plan to further develop the inverse problem of D3G into a reinforcement learning paradigm. Instead of based on demonstrations, robots will learn coordination strategies through self-explorations.



\section*{Appendix}
\subsection{Forward Problem: Distributed Nash Seeking}

A key step for both Algorithms \ref{Algorithm_DIDG} and \ref{Algorithm_DPMPsolver} is to compute the Nash equilibrium of the dynamic game \eqref{eq_OCP} with current parameter $\Theta^k$.
Distributed Nash equilibrium-seeking algorithms for general-sum games have been established in existing literature~\cite{YM-HG:17, tatarenko2020geometric,ZY-HB-MZ-RW:21}. 
Here, we directly employ the result in~\cite{tatarenko2020geometric} for distributed iterative Nash seeking. Note that~\cite{tatarenko2020geometric} is only applicable for solving unconstrained Nash equilibrium. To address this, we follow the principle in~\cite{LSM:06} which leverages the KKT condition in \eqref{eq_kkt} to convert the constrained problem into an unconstraint version. The algorithm is summarized in Algorithm \ref{Algorithm_GAME_shooting}.

\begin{algorithm2e}
	\setstretch{1.0}
	\caption{Distributed Nash-Equilibrium seeking for Dynamic Games, the local update for robot $i$.}
	\label{Algorithm_GAME_shooting}
	\SetAlgoLined
	\textbf{Input}  $\mathbf{P}_i(\theta_i)$, $x_i^0$.\\
	\textit{Initialize} a random guess for $\mathbf{u}_{i}^{\tau=0}$\\
	
 \While{$\max(|\mathbf{u}_{i}^{\tau+1}-\mathbf{u}_{i}^{\tau}|)\ge\epsilon_{\mathbf{u}}$ }{
 {Represent} $\mathbf{x}_{i}^\tau$ as a function of $\mathbf{u}_{i}^\tau$ using equation \eqref{eq_kkt_c} and the initial condition $x_i^{t=0}$.\\
 {Represent} $\bm{\lambda}_{i}$ as a function of $\mathbf{x}_{i}^\tau$ and $\mathbf{u}_{i}^\tau$ using equation \eqref{eq_kkt_a}, which can then be further represented as a function of only $\mathbf{u}_{i}^\tau$. \\
 Eliminate variables $\mathbf{x}_{i}$ and $\bm{\lambda}_{i}$ in \eqref{eq_def_hi} by their representation of $\mathbf{u}_{i}^\tau$, and obtain a reformulated $\bar{\textbf{H}}_i$ as a function of $\mathbf{u}_{i}^\tau$. \label{step_eliminate}\\
 Use $\bar{\textbf{H}}_i$ and \eqref{eq_kkt_b} to compute $\nicefrac{d\bar{\textbf{H}}_i}{d\mathbf{u}_{i}^\tau}$.\\
 \label{step_s_update}
 State update:
 $\displaystyle\mathbf{u}_{i}^{\tau+1} = \mathbf{u}_{i}^{\tau}-\kappa\frac{d\bar{\textbf{H}}_i}{d\mathbf{u}_{i}^\tau}$
 }
	${\mathbf{u}^{\star}_{i}}=\mathbf{u}_{i}^{\tau}$;~
	${\mathbf{x}^{\star}_{i}}=\mathbf{x}_{i}^{\tau}$ \\
	\textbf{Output} $\bm{\xi}_i^{\star}=\{{\mathbf{u}^{\star}_{i}},{\mathbf{x}^{\star}_{i}}\}$.
\end{algorithm2e}
Note that for step \eqref{step_eliminate}, we do not eliminate variable $\mathbf{x}_{\mathcal{N}_i}$ because it does not depend on $\mathbf{u}_{i}$. In addition, comparing the state update in step \eqref{step_s_update} with that in~\cite{tatarenko2020geometric}, the agent's objective functions are only coupled among neighboring agents, thus, the consensus step that appeared in~\cite{tatarenko2020geometric} can be omitted. Finally, we note that the geometric convergence of~\cite{tatarenko2020geometric} is ensured with a sufficiently small step-size $\kappa$ only if $\bar{\mathbf{H}}_i$ is strictly convex in $\mathbf{u}_{i}$. This can be ensured by Assumption \ref{Ass_func} if further assuming that robots have affine dynamics functions. For general non-linear functions, there has been no \textit{theoretical} convergence guarantee. However, in all of our experiments, the convergence of Algorithm \ref{Algorithm_GAME_shooting} is observed.

\subsection{Proof of Lemma \ref{LM_alg3}}

	The establishment of Algorithm \ref{Algorithm_DPMPsolver} is based on one of our previous works for solving coupled linear constraints using distributed network flows~\cite{XSB19TAC,XS18ACC}. The following update 	
	\begin{subequations}\label{eq_gradientsolver}
	\begin{align}\label{eq_dist_DPMPsol}
		&\bm{Y}_i^{\tau+1}=\bm{Y}_i^\tau-\delta\bm{\Psi}_i^{\top}\bm{v}_i^\tau\\
		&{\bm{Z}}_i^{\tau+1}={\bm{Z}}_i^\tau+\delta\bm{v}_i^\tau	\\
  &\bm{v}_i^\tau=\bm{\Psi}_i\bm{Y}_i^\tau-\widehat{\bm{C}}_i-\sum_{\ell\in\mathcal{N}_i}(\bm{Z}_{i}^\tau-\bm{Z}_\ell^\tau)
	\end{align}
	\end{subequations}
	is a first order discretization of the algorithm in~\cite{XS18ACC}. With a proper choice of $\delta$, the convergence of update \eqref{eq_gradientsolver} is exponential and has been theoretically certified. Building on this result, here, we only need to verify that the equilibrium of \eqref{eq_gradientsolver} solves problem \eqref{eq_GLE}.
Specifically, the equilibrium of \eqref{eq_gradientsolver} implies $\bm{v}_i^{\star}=\bm{\Psi}_i\bm{Y}_i^{\star}-\widehat{\bm{C}}_i-\sum_{\ell\in\mathcal{N}_i}(\bm{Z}_{i}^{\star}-\bm{Z}_\ell^{\star})=0$ for all $i\in\{1,\cdots,m\}$. It follows that
\begin{align}\label{eq_sumv}
	\sum_{i=1}^{m} \bm{v}_i^{\star}=\!\sum_{i=1}^{m}\left(\bm{\Psi}_i\bm{Y}_i^{\star}+\widehat{\bm{C}}_i\right)-\!\sum_{i=1}^{m}\left(\sum_{j\in\mathcal{N}_i}(\bm{Z}_{i}^{\star}-\bm{Z}_j^{\star})\right)=0
\end{align}
Since the network is undirected, one has
\begin{align*} \sum_{i=1}^{m}\sum_{j\in\mathcal{N}_i}\bm{Z}_j^{\star}=\sum_{j=1}^{m}\sum_{i\in\mathcal{N}_j}\bm{Z}_j^{\star}=\sum_{j=1}^{m}|\mathcal{N}_j|\bm{Z}_{j}^{\star}\\
=
\sum_{i=1}^{m}|\mathcal{N}_i|\bm{Z}_{i}^{\star}=\sum_{i=1}^{m}\sum_{j\in\mathcal{N}_i}\bm{Z}_{i}^{\star}
\end{align*}
This and equations \eqref{eq_sumv} yield \eqref{eq_GLE}. \hfill \qed

\bibliographystyle{IEEEtran}

\bibliography{bibs/Ref_main}

\end{document}